\documentclass[12pt]{article}
\usepackage{amsmath}
\usepackage{graphicx}
\usepackage{enumerate}
\usepackage{natbib}
\usepackage{academicons}
\usepackage{orcidlink}
\usepackage{mismath} 
\newcommand{\blind}{1}

\usepackage{booktabs} 
\usepackage{printlen} 
\usepackage{graphics}
\usepackage{amsfonts} 
\usepackage{amsthm} 
\usepackage{dsfont} 
\usepackage{authblk} 
\usepackage{verbatim} 
\usepackage[normalem]{ulem} 

\usepackage{xcolor}
\hypersetup{
    colorlinks,
    linkcolor={red!50!black},
    citecolor={blue!50!black},
    urlcolor={blue!80!black}
}

\newtheorem{theorem}{Theorem}
\newtheorem{corollary}{Corollary}
\newtheorem{lemma}{Lemma}

\newcommand{\varhyphen}[1]{{\operatorname{\mathit{#1}}}}

\def\spacingset#1{\renewcommand{\baselinestretch}{#1}\small\normalsize}
\spacingset{1}

\newcommand{\defeq}{\stackrel{\text{def}}{=}}

\begin{document}


\if1\blind
{
  \title{\bf A Fisher's exact test justification of the TF–IDF term-weighting scheme}
  \author{Paul Sheridan\,\orcidlink{0000-0002-5484-1951}\,\thanks{Corresponding author: \texttt{paul.sheridan.stats@gmail.com} (P. Sheridan)}}
  \author[1]{Zeyad Ahmed\,\orcidlink{0009-0003-0781-2968}\,}
  \author[2,3]{Aitazaz A. Farooque\,\orcidlink{0000-0002-5353-6752}\,}

  \affil[1]{School of Mathematical and Computational Sciences, University of Prince Edward Island, Charlottetown, PE, Canada}
  \affil[2]{Canadian Centre for Climate Change and Adaptation, University of Prince Edward Island, St Peters Bay, PE, Canada}
  \affil[3]{Faculty of Sustainable Design Engineering, University of Prince Edward Island, Charlottetown, PE, Canada}
  \date{}
  \renewcommand\Affilfont{\itshape\small}
  \maketitle
} \fi

\if0\blind
{
  \bigskip
  \bigskip
  \bigskip
  \begin{center}
    {\LARGE\bf A Fisher's exact test justification of the TF–IDF term-weighting scheme}
\end{center}
  \medskip
} \fi

\bigskip
\begin{abstract}
Term frequency–inverse document frequency, or TF–IDF for short, is arguably the most celebrated mathematical expression in the history of information retrieval. Conceived as a simple heuristic quantifying the extent to which a given term's occurrences are concentrated in any one given document out of many, TF–IDF and its many variants are routinely used as term-weighting schemes in diverse text analysis applications. There is a growing body of scholarship dedicated to placing TF–IDF on a sound theoretical foundation. Building on that tradition, this paper justifies the use of TF–IDF to the statistics community by demonstrating how the famed expression can be understood from a significance testing perspective. We show that the common TF–IDF variant TF–ICF is, under mild regularity conditions, closely related to the negative logarithm of the $p$-value from a one-tailed version of Fisher's exact test of statistical significance. As a corollary, we establish a connection between TF–IDF and the said negative log-transformed $p$-value under certain idealized assumptions. We further demonstrate, as a limiting case, that this same quantity converges to TF–IDF in the limit of an infinitely large document collection. The Fisher's exact test justification of TF–IDF equips the working statistician with a ready explanation of the term-weighting scheme's long-established effectiveness.
\end{abstract}

\noindent%
\vfill

\newpage
\spacingset{1.1}


\section{Introduction}\label{sec:intro}
\emph{Term frequency–inverse document frequency} (TF–IDF) is a classical formula quantifying the extent to which a given term of interest stands out in any one given document out of many. It is famously the template for numerous variants. In the now more than fifty years since it was introduced by~\cite{Salton1973}, TF–IDF has acquired broad acceptance among information retrieval practitioners in particular and text analysts in general. Its versatility and effectiveness make TF–IDF a fundamental technique for such core text analysis tasks as keyword extraction, text summarization, automatic indexing, document classification, and document retrieval.

While TF–IDF was initially conceived as a heuristic, there is a long tradition, dating back at least  to~\cite{Joachims1997}, of seeking to establish a sound theoretical basis for its use. In his comprehensive work on the subject, \cite{Roelleke2013} categorizes existing motivations for TF–IDF according to four theoretical frameworks: information theory~\citep[e.g.,][]{Aizawa2003}, probabilistic relevance modeling~\citep[e.g.,][]{Wu2008}, statistical language modeling~\citep[e.g.,][]{Elkan2005}, and the divergence from randomness paradigm~\citep[e.g.,][]{Amati2002}. From each standpoint, rationales of varying persuasiveness for TF–IDF have been advanced. Owing to these successes, TF–IDF is no longer to be regarded as a heuristic per se. Existing justifications, however, primarily target the information retrieval community. This underscores a gap in the literature on TF–IDF foundations, where explanatory frameworks drawn from other areas of text analysis remain to be explored.
 
This paper examines TF–IDF from the novel perspective of statistical significance testing. In particular, we demonstrate that TF–ICF, a common TF–IDF variant, is closely related to a one-tailed version of Fisher’s exact test of statistical significance. Fisher’s exact test is commonly used in bioinformatics research, often under the name of the hypergeometric test, to identify statistically under- or over-represented genes in lists of genetic pathways~\citep[see][]{Maleki2020}. The use of statistical tests for the analysis of text is likewise well-established~\citep[see][]{Dunning1994}. Here, we show that the TF–ICF variant arises, subject to mild regularity conditions, as a key component in the negative logarithm of the said one-tailed Fisher’s exact test $p$-value. The further imposition of the following idealized conditions yields a direct relationship between the negative log-transformed $p$-value and TF–IDF:~(1) all documents are of equal length;~(2) the term of interest appears in the document under consideration; and~(3) in all other documents, the term appears either not at all or with the same fixed positive value as in the document of interest. Finally, we find that TF–IDF equals the negative log-transformed test $p$-value, as the number of documents in a collection tends to infinity, in the limiting case when~(1) document lengths are constant;~(2) the term of interest appears exclusively in the document under consideration; and~(3) the term appears either exclusively or not at all in all other documents.

The primary contribution of this paper is to establish a theoretical underpinning for TF-IDF based on the theory of statistical inference. We are motivated by the recent work of~\cite{Sheridan2024} who show that TF–IDF performs similarly to Fisher’s exact test when compared on a host of standard text analysis tasks, and call attention to a possible mathematical connection between the two procedures. This paper builds on these empirical findings by making the connection explicit.

The rest of the paper is organized into five sections. Section~\ref{sec:background} provides necessary background information. In Subsection~\ref{subsec:notation}, we introduce the bag-of-words model notation used throughout the paper for representing a collection of documents. Subsection~\ref{subsec:term-weightings} concerns term-weighting schemes. TF-IDF and TF-ICF are defined. Some well-established limitations of term-weightings are acknowledged. The meteoric rise of word-embeddings, a powerful alternative to term-weightings, is discussed. In Subsection~\ref{subsec:fisher-test}, we introduce Fisher's exact test in a text analysis context, and elucidate how the negative log-transformed test $p$-value defines a term-weighting scheme. The use of significance tests in functional genomics is reviewed. Section~\ref{sec:related-work} overviews the multiple existing theoretical justifications for TF–IDF and situates our significance testing approach within that landscape. Specifically, we motivate how our Fisher's exact test justification of TF-IDF complements existing ones for this fundamental term-weighting scheme. In Section~\ref{sec:main-result}, we present our main results. After establishing a lemma, we demonstrate a connection between TF-ICF and Fisher's exact test through Theorem~\ref{theorem:hgt-tficf}. Direct connections between TF-IDF and the test are established through Corollaries~\ref{cor:hgt-tficf-1} and~\ref{cor:hgt-tficf-2}. However, we emphasize that the corollaries hold under restrictive conditions that are idealized compared to real-world scenarios. Some numerical examples are provided in Section~\ref{sec:numerical-examples} as a sanity check on our results. Section~\ref{sec:conclusion} concludes the paper with a restatement of our main findings and a discussion of some potential directions for future work. We hope this work will help motivate the use of TF-IDF to the statistics community by explaining its long-standing effectiveness from a statistical perspective.

\section{Background} \label{sec:background}
This section introduces the notation used throughout the rest of the paper, defines TF–IDF and related term-weighting schemes, and describes the application of Fisher's exact test in text analysis.

\subsection{Modeling framework} \label{subsec:notation}
We adopt the well-known bag-of-words model for representing a collection of $d>0$ documents composed from a vocabulary of $m>0$ distinct terms (denoted by $t_1, \ldots, t_m$). In the model, each document is regarded as a multiset (or ``bag'') of its constituent terms (or ``words''). Term order is consequently ignored. We employ $d_1, \ldots, d_d$ to denote both the collection documents and their associated multiset representations, with the meaning clear from the context. Identifying the terms in each document with their respective multiplicities (or counts) completes the representation. \emph{Term frequency} (TF), which we shall denote by $TF(i,j)\defeq n_{ij}\geq 0$, is synonymous with the multiplicity of the $i$'th term ($1\leq i\leq m$) term in the $j$'th document ($1\leq j\leq d$). It is convenient to arrange the TFs in an $m\times d$ matrix, $[n_{ij}]_{m \times d}$, where the $(i,j)$'th entry is the number of times term~$t_i$ occurs in document~$d_j$. This matrix is known as a \emph{term-document matrix}, and all of the usual bag-of-words model quantities derive from its values. In machine learning, terms are treated as features, while documents are considered observations. Traditionally, columns are reserved for features and rows for observations, making the transpose of the term-document matrix a natural choice for machine learning contexts. Not abiding by this convention facilitates our exposition. Table~\ref{tab:bag-of-words-model} describes the bag-of-words notational apparatus employed in this work.

\begin{table*}[th]
\centering
\begin{tabular}{p{3cm}p{12cm}}
Symbol & Description \\
\midrule
$d>0$ & Number of documents in the collection \\
$m>0$ & Number of distinct terms out of which collection documents are composed \\
$t_i$ ($1\leq i\leq m$) & The $i$'th term in the vocabulary $\left\{t_1,\ldots,t_m\right\}$ \\
$d_j$ ($1\leq j\leq d$) & Multiset representation of the $j$'th document in the collection \\
$b_{ij}$ & Indicator function defined as $1$ if the $i$'th term occurs in the $j$'th document, and $0$ otherwise \\
$b_i=\sum_{j=1}^d b_{ij}$ & Number of documents containing the $i$'th term \\
$n_{ij}$ & Occurrences of the $i$'th term in the $j$'th document \\
$n_i = \sum_{j=1}^d n_{ij}$ & Occurrences of the $i$'th term in the collection \\
$n_j = \sum_{i=1}^m n_{ij}$ & Number terms in the $j$'th document \\
$n = \sum_{j=1}^d n_j$ & Number of terms in the collection \\
$p_{ij} = n_{ij} / n_j$ & Proportion of $i$'th term occurrences in the $j$'th document \\
$p_i = n_i / n$ & Proportion of $i$'th term occurrences in the collection \\
$\tilde{p}_{ij}=\frac{n_i-n_{ij}}{n-n_j}$ & Proportion of $i$'th term occurrences not in the $j$'th document \\
$\check{p}_{ij}=p_{ij} + 1/n_j$ & A quantity specifically used in the proof of Lemma~\ref{lemma:quotient-term} \\
\midrule
\end{tabular}%
\caption{Notation for the bag-of-words model representation of a collection of documents.}
\label{tab:bag-of-words-model}
\end{table*}

\subsection{TF–IDF family term-weighting schemes} \label{subsec:term-weightings}
A \emph{term-weighting scheme} is any rule assigning a nonnegative, real-valued weight to every combination of term and document within a collection thereof. The weight of a term in a document is intended to reflect the importance of that term in representing the document's content relative to the entire  collection. Term weights are extensively used in traditional NLP methods to categorize documents based on subject matter, as well as a wide variety of other attributes, including sentiment, toxicity, stance, valence, humor, among others as discussed in recent survey papers by~\cite{Li2022} and~\cite{Fields2024}. Notably, an early implementation of such methods is found in~\cite{Elkan2007}, a patent filed in 2001 which details a system that assesses document quality by analyzing low-level textual features to infer such high-level attributes as humor and sentiment. Text classification represents one set of techniques that can be applied to the much broader challenge of natural language understanding.

Term frequency may be regraded as a rudimentary term-weighting scheme. Its earliest application as such in a modern information retrieval context is attributed to~\cite{Luhn1957}. The TF scheme equates the importance of a term in a document with the frequency of its occurrence in that document. TF is, however, not ideally suited for use as a term-weighting. To glean why, we need only consider that the most frequently occurring terms in a document (e.g., articles, prepositions, and conjunctions) cannot be expected to provide much information about its content. 

To address this limitation of TF, \cite{SparckJones1972} proposed a mild variant of the formula which is known today as \emph{inverse document frequency} (IDF).\footnote{Sp{\"a}rck Jones used an IDF variant equivalent to $\log(d/b_i) + 1$ in her 1972 publication.} It is defined as $IDF(i)\defeq \log(d/b_i)$ where $\log(x)$, here and throughout the paper, denotes the natural logarithm of a positive real number,~$x$. In this case, $x=d/b_i$ is strictly positive since~$b_i$ (i.e., the number of documents containing the $i$'th term) falls within the range of~$1$ to~$d$. Note that the IDF formula, along with all other formulas presented in this paper, remains valid regardless of the chosen base. IDF quantifies how rare or common a term is across all the documents in a collection. \cite{Salton1973} soon thereafter coupled TF with IDF as
\begin{equation} \label{eq:tfidf}
\varhyphen{TF-IDF}(i,j) \defeq n_{ij} \log(d / b_i)
\end{equation}
to form the much renowned TF–IDF scheme. The TF component captures the importance of a term within a single document. Meanwhile, IDF penalizes terms appearing frequently across the collection. Terms with high TF–IDF scores are those that are frequent within a document but rare across the collection, thus quantifying their importance in distinguishing any single document from among other collection documents.

The 1990s marked the onset of a proliferation of TF–IDF variants that has continued down to the present day. Term-weightings in the TF–IDF family adhere to a common template. This template, articulated in~\cite{Dogan2019}, is a product of three factors: a document level weight (e.g., TF), a collection level weight (e.g., IDF), and a document length normalization factor (e.g., $L^2$ normalization). Of the many TF–IDF variants, the TF-ICF term-weighting introduced by~\cite{Kwok1990}
\begin{equation} \label{eq:tficf}
\varhyphen{TF-ICF}(i,j) \defeq n_{ij} \log(n / n_i)
\end{equation}
features prominently in this work. The \emph{inverse collection frequency} (ICF) factor, likewise introduced by~\cite{Kwok1990}, is an IDF cousin defined as $ICF(i) \defeq \log(n/n_i)$. Note that ICF is well-defined since the logarithm argument is strictly positive, given that $n/n_i$ falls within the range of~$1$ to~$n$. Regarding TF–IDF variants, much more could be said. We refer the reader to \cite{Roelleke2013},~\cite{Dogan2019}, and~\cite{Alshehri2023} for comprehensive reviews on the subject.

Although the $IDF(i)$ and $ICF(i)$ functions closely resemble each other, they accept fundamentally different arguments as inputs, resulting in distinct interpretations. $IDF(i)$ is a decreasing function of the inverse proportion $d/b_i$ of documents containing the $i$'th term. The IDF value of a term, therefore, reflects its rarity across a collection of documents. By contrast, $ICF(i)$ is a decreasing function of the inverse proportion $n/n_i$ of the $i$'th term's occurrences relative to all term occurrences in the collection. Accordingly, the ICF value of a term reflects its rarity in the collection as a whole without regard to document level dispersion patterns.

Term-weightings have several widely acknowledged drawbacks. First, terms are represented as sparse, high-dimensional vectors of length matching the vocabulary size. This is memory intensive and computationally inefficient for large vocabularies. Second, semantic relationships between terms are not easily captured. For instance, the terms ``bull'' and ``heifer'' are treated as separate entities despite both pertaining to the bovine world. Third, the order of terms within documents is disregarded, severely hampering the modeling of subtle contextual nuances. Word embeddings, and transformer models more broadly, triumphantly meet these challenges.

Word embeddings represent terms as dense, low-dimensional vectors of fixed-length, typically spanning a few hundred dimensions. They are learned from large document collections using unsupervised learning techniques. The past decade has seen widespread application of word embeddings in NLP applications with startling success. \cite{Incitti2023} identify Word2Vec, GloVe, fastText, ELMo, and BERT as commonly used word embeddings. Word2Vec, GloVe, and fastText word embeddings are ``static'' in the sense that each word in a vocabulary is represented as a unique vector. ELMo and BERT word embeddings are ``contextualized'', meaning that terms are represented by different vectors depending on context. BERT, being built on the transformer architecture, is a rudimentary large language model (LLM). Its success sparked the development of various other large language, transformer-based models, including RoBERTa, ALBERT, and DeBERTa. BERT also influenced the development of generative pre-trained transformer (GPT) models. GPT architectures, notably ChatGPT and DeepSeek, use embeddings of tokens, which are roughly akin to syllables, rather than of full words. Today term-weightings are increasingly being supplanted by word embeddings and transformer models, which better capture textual meaning and nuance, due to their state-of-the-art (SOTA) performance on diverse NLP tasks~\citep[see][]{Islam2023}.

One might presume that word embeddings and transformer models have rendered TF–IDF-like weightings obsolete. But this is not strictly the case. As highlighted by \cite{Rathi2023}, attempts have been made to enhance word embeddings with term-weightings. \cite{Guo2019}, \cite{Kim2019}, and \cite{Xie2022} devised hybrid approaches of this kind for improved text classification. \cite{Whaba2023} report a support vector machine trained on TF–IDF features to be a comparable, cost effective, and readily interpretable alternative to transformer models (e.g., RoBERTa) for classifying certain domain-specific texts. In a similar vein, \cite{Galke2023} find that a multi-layer perceptron trained on TF–IDF derived features offers a viable alternative to SOTA transformer models for single- and multi-label text classification tasks when computing power is limited. They also emphasize the importance of rigorously testing SOTA text classifiers against strong bag-of-words model baselines. In a recent benchmark study, \cite{Reusens2024} reached similar conclusions, expanding on \cite{Galke2023} by comparing SOTA models with simpler approaches across five different text classification tasks. These findings suggest that while TF–IDF-like weightings are no longer at the forefront of NLP, they can still be effective in specific, niche roles within text classification.

\subsection{Fisher's exact test} \label{subsec:fisher-test}
Fisher's exact test of statistical significance was developed by R.~A. Fisher~\citep{Fisher1934, Fisher1935a} to quantify the strength of association between two dichotomous categorical variables. The qualification ``exact'' indicates that the test relies on an exact sampling distribution as opposed to an asymptotic approximation of one (cf. Pearson's $\chi^2$ test). In its one-tailed form, Fisher's exact test is equivalent to the hypergeometric test~\citep[see][]{Rivals2007}.

Fisher's exact test and related procedures are frequently encountered in the analysis of contingency tables across the sciences. In functional genomics, Fisher's exact test is a well-established procedure for assessing whether genes/proteins are statistically under- or over-represented in a query list relative to a background list~\citep[see][]{Huang2009}. \cite{Maleki2020} curate over 100 such methods and tools based on Fisher's exact test and its counterparts. A current trend involves applying significance tests to calculate over-representation scores for various other biological annotations of interest. For example, \cite{Piron2024} propose a test for transcript enrichment. \cite{Garcia2022} employ a testing framework based on a generalization of the hypergeometric distribution for the over-representation analysis of regulatory elements.

Gene over-representation analysis methods are complemented by transformer-based models, such as AlphaFold~\citep[see][]{Jumper2021}, which are capable of predicting~3D protein structures from amino acid sequences with accuracies comparable to those obtained using experimental methods. Often, the functions of genes identified as over-represented are are unknown. In such cases, AlphaFold and related methods can aid in inferring protein function by predicting structures and enabling comparisons with known proteins. These structural insights can, in turn, enhance the accuracy and interpretability of gene over-representation analyses~\citep[see][]{Medvedev2023}. In this way, AlphaFold and gene over-representation analysis work together in going from from list of genes, to enriched functions, to 3D protein structures, to biological insights.

Fisher's exact test is readily applicable in text analysis. Consider the claim that the occurrences of a specific term, $t_i$ ($i=1,\ldots,m$), are disproportionately concentrated in a specific document, $d_j$ ($j=1,\ldots,d)$, within a collection of documents. Table~\ref{tab:contingency-table} shows the occurrences of the term $t_i$ arranged in a contingency table. Fisher's exact test can be used to assess the significance of a difference between the corresponding proportions: the probability $\pi_0$ that term $t_i$ belongs to document $d_j$ (cf.~$p_{ij}=n_{ij}/n_j$), and the probability $\pi_1$ that term $t_i$ belongs to a document other than $d_j$ (cf.~$\tilde{p}_{ij}=(n_i-n_{ij})/(n-n_j)$). In a one-tailed application of the test, under the null hypothesis $\mathcal{H}_0: \pi_0 = \pi_1$, the hypergeometric distribution tail probability
\begin{equation} \label{eq:fisher-test-pvalue}
H_{ij} \defeq H(n_{ij};n_i,n_j,n) = \sum_{k \geq n_{ij}} \frac{\binom{n_i}{k} \binom{n-n_i}{n_j-k}}{\binom{n}{n_j}}
\end{equation}
yields a $p$-value associated with observing an outcome at least as extreme as the observed one. The $i$'th term is said to be \emph{over-represented} (or \emph{enriched}) in the $j$'th document should the calculated $p$-value fall below a preselected significance threshold.

\begin{table*}[th]
\centering
\begin{tabular}{lccl}
 & \multicolumn{2}{c}{Document of interest, $d_j$} &  \\
\cmidrule(rl){2-3}
Term of interest, $t_i$ & Terms $\in$ $d_j$ & Terms $\notin$ $d_j$ & Row total \\
\midrule
Terms $=$ $t_i$ & $n_{ij}$ & $n_i - n_{ij}$ & $n_i$ \\
Terms $\neq$ $t_i$ & $n_j - n_{ij}$ & $n - n_i - (n_j - n_{ij})$ & $n - n_i$ \\
Column total & $n_j$ & $n - n_j$ &  \\
\midrule
\end{tabular}
\caption{Association between a specific term and a specific document in the context of a collection of documents.}
\label{tab:contingency-table}
\end{table*}

An associated term-weighting is naturally defined as the negative logarithm of the test $p$-value: $-\log H_{ij}$. Since $p$-values lie inherently in the interval $(0,1]$, it follows that $-\log H_{ij}$ meets the criteria for a term-weighting function, as the transformation guarantees that term weight is nonnegative and increases with the degree of over-representation. We demonstrate in Section~\ref{sec:main-result} that the quantity $-\log H_{ij}$ is fundamentally connected to the TF–ICF and TF–IDF term-weighting functions.

\section{Related work} \label{sec:related-work}
This section surveys the current crop of theoretical justifications for TF–IDF. We direct the interested reader to~\cite{Sheridan2024} for an expanded discussion on the subject. For the already initiated, we recommend the monumental treatment of TF–IDF foundations that is~\cite{Roelleke2013}.

A substantial amount of scholarly work exists on the theoretical foundations of TF–IDF. One notable approach seeks to ground TF–IDF in information-theoretic arguments. \cite{Aizawa2003}, for instance, links TF–IDF to the mutual information between terms and documents, treating both as realizations of random variables. \cite{Roelleke2013} shows how a TF–IDF variant is expressible as a difference of two Kullback--Leibler divergences. Broadly speaking, information-theoretic justifications treat IDF as a proxy for a term’s information content, and stress the highly informative nature of rare terms. A central insight of these approaches is that natural language is structured and that terms vary in their ability to distinguish meaning.

Probabilistic relevance modeling is another noteworthy source of TF–IDF justifications~\citep[see][]{Joachims1997, deVries2005, Roelleke2006, Roelleke2008, Wu2008}. These models estimate the probability that a document in a collection is relevant to a given query~\citep[see][]{Robertson1976, Robertson2009}. Documents are ranked in decreasing order of their relevance. TF–IDF justifications rooted in this modeling framework treat TF and IDF as components in estimating the likelihood of a document being relevant to a query under a certain probability model of language. They often yield TF–IDF-like expressions as approximations or special cases. In a similar vein to information-theoretic approaches, probabilistic relevance modeling-based justifications leverage the nonrandom nature of word usage in natural language, particularly the tendency of terms to recur within documents where they initially appear (i.e., the word burstiness phenomena).

Probabilistic language models that are not explicitly incorporated within a relevance modeling framework constitute their own class of justification for TF–IDF. \cite{Hiemstra2000} obtains a term-weighting with TF–IDF characteristics by applying a logarithmic transform to an estimator for the conditional probability of a term given a document. This key result was expanded on by~\cite{Roelleke2008} and~\cite{Roelleke2013}. \cite{Elkan2005} shows that an asymptotic approximation to the Dirichlet-multinomial distribution Fisher kernel incorporates terms comparable to the TF (specifically, its logarithm) and IDF components of TF–IDF. Elkan's Fisher kernel should not be confused with Fisher's exact test, as they are entirely different statistical concepts. \cite{Amati2002} derive TF–IDF, under certain regularity conditions, from a simple probabilistic model that assumes term frequencies are binomially distributed within documents. This result was rediscovered by~\cite{Havrlant2017}. \cite{Sunehag2007} obtains a TF–IDF variant as a cross-entropy related to a probabilistic language model, which was introduced by~\cite{Elkan2005}, that accounts for word burstiness. As with probabilistic relevance modeling, a core feature shared by these approaches is their reliance, in one way or another, on the burstiness inherent in natural language.

The significance testing approach taken in this work aligns closely with those justifications for TF–IDF rooted in probabilistic language modeling. \cite{Sheridan2024} show that TF–IDF exhibits similar performance to the Fisher's exact test term-weighting scheme on standard text classification tasks. This paper builds on these empirical findings by establishing a mathematical link between Fisher’s exact test and TF–IDF. The significance testing perspective highlights over-representation and statistical surprise, treating a term as important if it appears in a document more frequently than expected by chance. As we will show, however, this approach is bound up with word burstiness.

\section{From Fisher's exact test to TF–IDF} \label{sec:main-result}
In this section, we relate Fisher's exact test to the TF–IDF term-weighting scheme. After taking pains to establish a preliminary result, we show that the negative logarithm of the $p$-value from Fisher's exact test can be expressed as a sum that includes TF–ICF among its components. This is achieved by means of elementary mathematical arguments. When certain restrictive assumptions are imposed, relations between the Fisher's exact test $p$-value and TF–IDF are obtained.

Previously, we used $H(n_{ij};n_i,n_j,n)$ to represent the tail probability of the hypergeometric distribution. Recalling the notation from Table~\ref{tab:bag-of-words-model}, we here let
\begin{align}
    h(n_{ij};n_i,n_j,n) &\defeq \binom{n_i}{n_{ij}} \binom{n-n_i}{n_j-n_{ij}} \Big/ \binom{n}{n_j} \nonumber \\
    b(n_{ij}; n_j, p_i) &\defeq \binom{n_j}{n_{ij}}p_i^{n_{ij}}(1-p_i)^{n_j - n_{ij}} \nonumber
\end{align}
denote the hypergeometric and binomial distribution probability mass function, respectively. Both $h(n_{ij};n_i,n_j,n)$ and $b(n_{ij}; n_j, p_i)$ are defined over the support $n_{ij} \in \left\{0,\ldots,n_j\right\}$. Note that $h(n_{ij};n_i,n_j,n)$ is the first term in the tail probability $H(n_{ij};n_i,n_j,n)$.

\begin{lemma} \label{lemma:quotient-term}
The quotient $Q_{ij} \defeq H(n_{ij}+1; n_i, n_j, n) / b(n_{ij}; n_j, p_i)$ satisfies $0<Q_{ij}<1$ when $p_i$ is sufficiently small, and $n_j$, $n_{ij}$, and $n_j - n_{ij}$ are sufficiently large.
\end{lemma}

\begin{proof}
It is evident that the quotient $Q_{ij}$ must be nonnegative. It therefore suffices to show that the difference
\begin{equation} \label{eq:log_quotient}
D_{ij} \defeq -\frac{1}{n_j}\log Q_{ij} = \frac{1}{n_j}\log b(n_{ij}; n_j, p_i) - \frac{1}{n_j}\log H(n_{ij}+1; n_i, n_j, n)
\end{equation}
is strictly positive under the stated conditions. The multiplicative factor $1/n_j$ is included as a mathematical convenience. We will see that doing so facilitates the exposition by enabling us to work with functions of $p_{ij}$, as opposed to functions of $n_{ij}$.

The starting point is to work each term in Eq.~(\ref{eq:log_quotient}) into a more manageable form. First, consider the binomial term $W_{b}(p_{ij}) \defeq \frac{1}{n_j} \log b(n_{ij}; n_j, p_i)$. Following up a little algebra with an application of Stirling's approximation, we obtain
\begin{align}
  W_{b}(p_{ij}) &= \frac{1}{n_j} \log\left(\binom{n_j}{n_{ij}}p_i^{n_{ij}}(1-p_i)^{n_j - n_{ij}}\right) \nonumber \\ 
  &= p_{ij}\log p_i + (1-p_{ij})\log (1 - p_i) + \frac{1}{n_j}\log\binom{n_j}{n_{ij}} \nonumber \\
  \begin{split}
  &= p_{ij}\log p_i + (1-p_{ij})\log (1 - p_i) + \frac{1}{n_j}\log n_j! - \frac{1}{n_{ij}}\log n_{ij}! - \frac{1}{n_j}\log\left(n_j-n_{ij}\right)!
  \end{split} \nonumber \\
  \begin{split} \label{eq:binomial-term}
  &= p_{ij}\log\left(\frac{p_i}{p_{ij}}\right) + (1-p_{ij})\log\left(\frac{1 - p_i}{1-p_{ij}}\right) -\frac{1}{2n_j}\log\left(2\pi n_j p_{ij} (1-p_{ij})\right) \\ 
  &\quad + \bigO(1/n_j) + \bigO(1/n_{ij}) + \bigO(1/(n_j-n_{ij})).
  \end{split}
\end{align}
To transition from the first line to the second, we separate the logarithm and use the relation $p_{ij} = n_{ij}/n_j$ to express $\frac{1}{n_j} \times n_{ij} \log p_i$ and $\frac{1}{n_j} \times (n_j - n_{ij}) \log(1 - p_i)$ as $p_{ij} \log p_i$ and $(1 - p_{ij}) \log(1 - p_i)$, respectively. In the third line, we express the binomial coefficient term from the second line as sum of three log factorial terms. In going from the third line to the fourth, we approximate any terms of the form $\log x!$ as $x\log x - x + \frac{1}{2}\log(2\pi x) + \bigO(1/x)$ for $x\gg 1$, where $x$ is an integer, and algebraically simplify the result.

Let us now turn to the hypergeometric cumulative term $\frac{1}{n_j} \log H(n_{ij}+1; n_i, n_j, n)$. \cite{Chvatal1979} established the inequality
\begin{equation*} \label{eq:chtaval}
H(n_{ij}; n_i, n_j, n) \leq \left(\left(\frac{p_i}{p_i + t}\right)^{p_i+t} \left(\frac{1-p_i}{1-p_i-t}\right)^{1-p_i -t}\right)^{n_j},
\end{equation*}
where $n_{ij}=(p_i+t)n_j$ with $t\geq 0$. For $H(n_{ij}+1; n_i, n_j, n)$, we find the corresponding inequality
\begin{equation} \label{eq:chtaval-specialized}
H(n_{ij}+1; n_i, n_j, n) \leq \left(\left(\frac{p_i}{\check{p}_{ij}}\right)^{\check{p}_{ij}} \left(\frac{1-p_i}{1-\check{p}_{ij}}\right)^{1-\check{p}_{ij}}\right)^{n_j}
\end{equation}
holds for $p_i \leq \check{p}_{ij}\leq 1$, where, it will be remembered, $\check{p}_{ij}=p_{ij}+1/n_j$. Applying a $\frac{1}{n_j}\log(\cdot)$ transform to the upper bound of Eq.~(\ref{eq:chtaval-specialized}) yields
\begin{align}
  W_{H}(p_{ij}) &\defeq \frac{1}{n_j} \log\left(\left(\left(\frac{p_i}{\check{p}_{ij}}\right)^{\check{p}_{ij}} \left(\frac{1-p_i}{1-\check{p}_{ij}}\right)^{1-\check{p}_{ij}}\right)^{n_j}\right) \nonumber \\
  \begin{split} \label{eq:hypergeom-term}
  &= \check{p}_{ij}\log\left(\frac{p_i}{\check{p}_{ij}}\right) + (1-\check{p}_{ij})\log\left(\frac{1-p_i}{1-\check{p}_{ij}}\right).
  \end{split}
\end{align}
To confirm that Eq.~(\ref{eq:hypergeom-term}) is well-defined, we must ensure that the argument of each logarithm is strictly positive. First, the quotient $p_i/\check{p}_{ij}$ is strictly positive since both $p_i$ and $\check{p}_i=p_i + 1/n_j$ are positive. Second, the quotient $(1-p_i)/(1-\check{p}_{ij})$ would be only non-positive if $n_{ij}=n_j$ which would result in $1-\check{p}_{ij}=1 - n_j/n_j - 1/n_j=-1/n_j$. However, this scenario is prevented by the requirement that $n_j - n_{ij}$ remains sufficiently large. Consequently, both logarithm arguments in Eq.~(\ref{eq:hypergeom-term}) are strictly positive, ensuring that the equation is well-defined. Note the pleasant resemblance between the leading terms of $W_{b}(p_{ij})$ and those terms making up $W_{H}(p_{ij})$.

Recall, we want to show that $D_{ij}$ is positive. It would suffice to demonstrate that
\begin{equation*} 
 W(p_{ij}) \defeq W_{b}(p_{ij}) - W_{H}(p_{ij}) > 0
\end{equation*}
since the $\frac{1}{n_j}\log H(n_{ij}+1; n_i, n_j, n)$ term in $D_{ij}$ is at most $W_{H}(p_{ij})$. But it is not the case that $W(p_{ij})>0$ holds true unconditionally.

Instead, we will content ourselves to show that $W(p_{ij})>0$ holds when $p_i$ is sufficiently small as compared with $p_{ij}$. In the supplementary materials, we show that $W(p_{ij})$ simplifies to
\begin{align*}
 W(p_{ij}) &= (p_{ij}-\check{p}_{ij})\log p_i - \left(p_{ij} + \frac{1}{2n_j}\right)\log p_{ij} -\left(1 - p_{ij} + \frac{1}{2n_j}\right)\log(1 - p_{ij}) \\
 &\quad  + \check{p}_{ij}\log\check{p}_{ij} + (1 - \check{p}_{ij})\log(1-\check{p}_{ij}) + (p_{ij}-\check{p}_{ij})p_i - \frac{1}{2n_j}\log\left(2\pi n_j\right) \\
 &\quad + \bigO(1/n_j) + \bigO(1/n_{ij}) + \bigO(1/(n_j-n_{ij})) + \lito(p_i)
\end{align*}
when any occurrences of $\log(1-p_i)$ are approximated by the first-order Taylor series expansion $\log(1-p_i) = -p_i + \lito(p_i)$. The leading term in the first line, $(p_{ij}-\check{p}_{ij})\log p_i=-\frac{1}{n_j}\log(n_i/n)=\frac{1}{n_j}\log(n/n_i)$, diverges to infinity as $p_i$ tends to $0$ (or equivalently, when $n$ grows large relative to $n_i$). Both of the remaining terms in the first line are positive. While the terms in the second line are negative, none possess sufficient magnitude to ultimately overcome the positive contribution of the $(p_{ij}-\check{p}_{ij})\log p_i$ term. Taken together, $\check{p}_{ij}\log\check{p}_{ij} + (1 - \check{p}_{ij})\log(1-\check{p}_{ij})$ is bounded by its global minimum of $-0.693$ which is achieved when $\check{p}_{ij}=0.5$. The $(p_{ij}-\check{p}_{ij})p_i$ term tends to $0$ as $p_i$ tends to $0$. The $-\frac{1}{2n_j}\log(2\pi n_j)$ term contributes a fixed negative value to $W(p_{ij})$. Meanwhile, the big~$\bigO$ and little~$\lito$ terms in the third line are negligible. Therefore $W(p_{ij})$ is greater than $0$ so long as $n_j$, $n_{ij}$, and $n_j - n_{ij}$ are sufficiently large, and $p_i$ is sufficiently small.

This concludes the proof.
\end{proof}

The following result establishes a connection between Fisher's exact test and the TF-ICF term-weighting function.

\begin{theorem} \label{theorem:hgt-tficf}
For the $i$'th term in the $j$'th document in a collection of documents, the Fisher's exact test $p$-value $H_{ij}$ is related to the TF-ICF term-weighting as
\begin{align}
\begin{split}
-\log H_{ij} &= \varhyphen{TF-ICF}(i,j) + n_{ij}\log p_{ij} + (n_j - n_{ij})(p_i - p_{ij}) - Q_{ij} + \bigO(1/n) \\
 &\quad + \bigO(\log n_j) + \bigO(\log n_{ij}) + \lito(p_{ij}) + \lito(p_i) + \lito\left(Q_{ij} + \bigO(1/n)\right)
\end{split} \label{eq:hgt-tficf} \\
-\log H_{ij} &\approx \varhyphen{TF-ICF}(i,j) + \Phi_{ij} \label{eq:hgt-tficf-approx}
\end{align}
with $\Phi_{ij} \defeq n_{ij}\log p_{ij} + (n_j - n_{ij})(p_i - p_{ij}) - Q_{ij}$ so long as $p_{ij}$ is sufficiently small, $p_i \ll p_{ij}$, and $n_j$ and $n_{ij}$ are sufficiently large.
\end{theorem}

\begin{proof}
The strategy is to simplify $-\log H_{ij}$ by using basic algebra coupled with classical approximation techniques.

As a first step, we approximate the first term $h(n_{ij};n_i,n_j,n)$ in the hypergeometric tail probability $H_{ij}$ and  with the binomial probability mass $b(n_{ij};n_j,p_i)$:
\begin{align}
-\log H_{ij} &= -\log\left(h(n_{ij}; n_i, n_j, n) + H(n_{ij}+1; n_i, n_j, n) \right) \nonumber \\
 &= -\log\left(b(n_{ij}; n_j, p_i) + H(n_{ij}+1; n_i, n_j, n) + \bigO(1/n)\right). \label{eq:binomial-approx-to-hypergeom}
\end{align}
In arriving at Eq.~(\ref{eq:binomial-approx-to-hypergeom}), we appeal to the fact that the hypergeometric probability mass function tends to the binomial one as $n\to\infty$ in such a manner that $n_i/n\to p_i$. \cite{Jaioun2014} supply a rigorous justification for the $\bigO(1/n)$ bound on the error term.

A basic algebraic manipulation yields
\begin{align*}
-\log H_{ij} &= -\log b(n_{ij}; n_j, p_i) - \log\left(1 + \frac{H(n_{ij}+1; n_i, n_j, n) + \bigO(1/n)}{b(n_{ij}; n_j, p_i)}\right) \\
 &= -\log b(n_{ij}; n_j, p_i) - \log\left(1 + Q_{ij} + \bigO(1/n)\right).
\end{align*}
In particular, we use the identity $\log(x + y) = \log(x) + \log(1 + y/x)$ with $x=b(n_{ij}; n_j, p_i)$ and $y=H(n_{ij}+1; n_i, n_j, n) + \bigO(1/n)$. In the second line, we have rewritten $H(n_{ij}+1; n_i, n_j, n)/b(n_{ij}; n_j, p_i)$ as $Q_{ij}$, and subsumed the constant factor $1/b(n_{ij}; n_j, p_i)$, i.e., constant with respect to $n$, into the $\bigO(1/n)$ term.

It follows from Lemma~\ref{lemma:quotient-term} that
\begin{align*}
-\log H_{ij} &= -\log b(n_{ij}; n_j, p_i) - Q_{ij} + \bigO(1/n) + \lito\left(Q_{ij} + \bigO(1/n)\right),
\end{align*}
where we have applied the Taylor series expansion $\log(1+x) = x + \lito(x)$ with $0<x=Q_{ij}+\bigO(1/n)<1$ for sufficiently large~$n$. As a consequence, the various regularity conditions required for Lemma~\ref{lemma:quotient-term} must be assumed from this point forward (i.e., $p_i$ is sufficiently small, and $n_i$, $n_{ij}$, $n_j-n_{ij}$ are sufficiently large).

Next, we find that the TF-ICF formula emerges naturally from working out the negative logarithm of $b(n_{ij}; n_j, p_i)$:
\begin{align}
-\log H_{ij} &= -\log b(n_{ij}; n_j, p_i) - Q_{ij} + \bigO(1/n) + \lito\left(Q_{ij} + \bigO(1/n)\right) \nonumber \\
\begin{split}
 &= -n_{ij}\log p_i - (n_j-n_{ij})\log(1-p_i) -\log\binom{n_j}{n_{ij}} - Q_{ij} \\
 &\quad + \bigO(1/n) + \lito\left(Q_{ij} + \bigO(1/n)\right) \nonumber
 \end{split} \\
 \begin{split}
 &= \varhyphen{TF-ICF}(i,j) -(n_j-n_{ij})\log(1-p_i) -\log\binom{n_j}{n_{ij}} - Q_{ij} \\
 &\quad + \bigO(1/n) + \lito\left(Q_{ij} + \bigO(1/n)\right).
\end{split} \label{eq:hgt-tficf-imtermediate}
\end{align}
We make use of the fact that $p_i=n_i/n$ to transition from the $-n_{ij}\log p_i$ term in the second line to the $\varhyphen{TF-ICF}(i,j)=n_{ij}\log(n/n_i)$ term in the third line. Note that this connection between term-weighting functions and the negative logarithm of the binomial distribution has previously been observed by \cite{Amati2002} and \cite{Havrlant2017}.

To complete the derivation, we approximate the $\log\binom{n_j}{n_{ij}}$ and $\log(1-p_i)$ components of Eq.~(\ref{eq:hgt-tficf-imtermediate}) and simplify the resulting expression algebraically:
\begin{align}
    -\log{H}_{ij} &= 
       \varhyphen{TF-ICF}(i,j)
       - (n_j - n_{ij})\log(1 - p_i) 
       - n_{j} \log n_j + n_{ij} \log n_{ij} \nonumber\\
    &\quad 
        + (n_{j}-n_{ij})\log{(n_{j}-n_{ij})} 
        + \bigO(\log{n_j})
        + \bigO(\log{n_{ij}})
        - Q_{ij} 
        + \bigO(1/n) \nonumber\\
    &\quad
        + \lito{(Q_{ij} + \bigO(1/n))}\nonumber\\
    &= \varhyphen{TF-ICF}(i,j) 
       - (n_j - n_{ij})\left(\log (1 - p_i) - \log(1-p_{ij})\right)
       - n_{ij} \log n_{j} 
       + n_{ij} \log n_{ij} \nonumber\\ 
    &\quad
       + \bigO(\log n_j)
       + \bigO(\log n_{ij})
       - Q_{ij}
       + \bigO(1/n) 
       + \lito{(Q_{ij} + \bigO(1/n))} \nonumber\\
    &= \varhyphen{TF-ICF}(i,j) 
       - (n_j - n_{ij})\left(\log(1 - p_i) - \log(1-p_{ij})\right)
       + n_{ij}\log{p_{ij}}
       - Q_{ij} \nonumber\\
    &\quad
       +\bigO(\log n_j)
       + \bigO(\log n_{ij})
       + \bigO(1/n)
       + \lito{(Q_{ij} + \bigO(1/n))} \nonumber\\
    &= \varhyphen{TF-ICF}(i,j)
       + n_{ij}\log{p_{ij}}
       + (n_j - n_{ij})\left(p_i -p_{ij}\right)
       - Q_{ij} 
       + \bigO(1/n) \nonumber\\
    &\quad
       + \bigO(\log n_j) 
       + \bigO(\log n_{ij})
       + \lito(p_{ij})
       + \lito(p_i)  
       + \lito{(Q_{ij} + \bigO(1/n))} \nonumber 
\end{align}
In the first line, we use Stirling's formula to approximate $\log\binom{n_j}{n_{ij}}$ as $n_j\log n_j - n_{ij}\log n_{ij} -(n_j-n_{ij})\log(n_j-n_{ij}) - \frac{1}{2}\log(n_j - n_{ij}) + \bigO(\log n_j) + \bigO(\log n_{ij})$. In going from the first line to the second, we express $\log(n_j-n_{ij})$ as $\log\left(n_j(1-n_{ij}/n_j)\right)=\log\left(n_j(1-p_{ij})\right)=\log n_j + \log(1-p_{ij})$, combine the resulting $\log(1-p_{ij})$ term with the like term $\log(1-p_i)$, and simplify the result. In the third line, we have combined the $-n_{ij}\log n_j$ and $n_{ij}\log n_{ij}$ terms from the second line as $n_{ij}(\log n_{ij}-\log n_j) = n_{ij}\log(n_{ij}/n_j)=n_{ij}\log p_{ij}$. In going from the third line to the final line, we apply a first-order Taylor expansion to $\log(1-p_{ij})$, approximating it as $-p_{ij} + \lito(p_{ij})$, and to $\log(1-p_i)$, approximating it as $-p_i + \lito(p_i)$. The final line matches Eq.~(\ref{eq:hgt-tficf}), and thus the theorem is proved.
\end{proof}

The following two results connecting Fisher's exact test with TF-IDF emerge from Theorem~\ref{theorem:hgt-tficf} as special cases. Recall that $b_{ij}$ is the indicator function, meaning that $b_{ij}=1$ if the $i$'th term occurs in the $j$'th document, and $b_{ij}=0$ otherwise.

\begin{corollary} \label{cor:hgt-tficf-1}
If a collection of documents satisfies
\begin{enumerate}
    \item $n_j = R$ for $1\leq j\leq d$ (i.e., documents are of equal length)
    \item $n_{ij}=r$ with $0< r \ll R$ (i.e., the $i$'th term occurs in the $j$'th document)
    \item $n_{i\ell}=r b_{i\ell}$ with $0\leq r \ll R$ for $\ell\neq j$ (i.e., if the $i$'th term occurs in the $\ell$'th document, then it occurs exactly $r$ times)
\end{enumerate}
then we have
\begin{align}
\begin{split}
-\log H_{ij} &= \varhyphen{TF-IDF}(i,j) - n_{ij}(1 - b_i/d)(1 - p_{ij}) - Q_{ij} + \bigO(1/d) \\
 &\quad + \bigO(\log n_j) + \bigO(\log n_{ij}) + \bigO(p_{ij}) + \lito(p_i) + \lito\left(Q_{ij} + \bigO(1/d)\right)
\end{split} \label{eq:hgt-tfidf-relaxed} \\
-\log H_{ij} &\approx \varhyphen{TF-IDF}(i,j) + \Psi_{ij} \label{eq:hgt-tfidf-approx}
\end{align}
with $\Psi_{ij} \defeq - n_{ij}(1 - b_i/d)(1 - p_{ij}) - Q_{ij}$ so long as $p_{ij}$ is sufficiently small, $p_i \ll p_{ij}$, and $n_j$ and $n_{ij}$ are sufficiently large.
\end{corollary}

The result is obtained by substituting variables into Eq.~(\ref{eq:hgt-tficf}) and simplifying. Interestingly, TF-ICF reduces to TF-IDF as a consequence of $n_i = r b_i$ and $n=Rd$. That we now have $\bigO(1/d)$ follows from $n=Rd$.

At last we arrive at a satisfyingly clean, albeit highly restrictive, connection between Fisher's exact test and TF-IDF.

\begin{corollary} \label{cor:hgt-tficf-2}
If a collection of documents satisfies
\begin{enumerate}
    \item $n_j = R$ for $1\leq j\leq d$ (i.e., documents are of equal length)
    \item $n_{ij} = R$ (i.e., the $i$'th term appears exclusively in the $j$'th document)
    \item $n_{i\ell}=R b_{i\ell}$ for $\ell\neq j$ (i.e., the $i$'th term appears either exclusively or not at all in the $\ell$'th document)
\end{enumerate}
then we have
\begin{equation} \label{eq:hgt-tfidf}
-\log H_{ij} = \varhyphen{TF-IDF}(i,j) + \bigO(1/d).
\end{equation}
\end{corollary}

This time we substitute $n_j=R$, $n_{ij}=R$, $n_i=Rb_i$, and $n=Rd$ directly into Eq.~(\ref{eq:hgt-tficf-imtermediate}). We observe that TF-ICF once again reduces to TF-IDF. In this extreme case, however, the $(n_j - n_{ij})\log(1-p_i)$, $\log\binom{n_j}{n_{ij}}$, and $Q_{ij}$ term vanish, leaving only the relatively clean formula $\varhyphen{TF-IDF}(i,j) + \bigO(1/d)$ behind. 

To show that $Q_{ij}$ vanishes we use its definition (see Lemma~\ref{lemma:quotient-term}). The numerator $H(R + 1; Rb_i, R , Rd)$ evaluates to a tail probability of zero because the number of occurrences of the $i$'th term, $R+1$, exceeds the length of the $j$'th document, $R$. Meanwhile, $b(R;R,p_i)=p_i^R$ is positive. Thus, $Q_{ij}$ is zero. Additionally, $\log\binom{n_j}{n_{ij}}$ vanishes since $n_j=n_{ij}=R$.

We conclude this section with a brief discussion on the terms $\Phi_{ij}$ and $\Psi_{ij}$ from Eq.~(\ref{eq:hgt-tficf-approx}) and Eq.~(\ref{eq:hgt-tfidf-approx}), respectively. The $\Phi_{ij}$ term, which equals $n_{ij}\log p_{ij} + (n_j - n_{ij})(p_i - p_{ij}) - Q_{ij}$, corrects for the discrepancy between $-\log{H_{ij}}$ and $\varhyphen{TF-ICF}(i,j)$. Interestingly, each of the first two components carry a natural interpretation. The $n_{ij} \log p_{ij}$ component captures how likely it is to observe the $i$'th term $n_{ij}$ times in the $j$'th document, assuming the local probability $p_{ij}$. The $(n_j - n_{ij})(p_i - p_{ij})$ component reflects the extent to which the rest of the document (i.e., the occurrences of terms other than the $i$'th term) diverges from what would be expected globally, adjusting for background mismatch. Together with $Q_{ij}$, these two components adjust for the discrepancy between $-\log{H_{ij}}$ and $\varhyphen{TF-ICF}(i,j)$. Similarly, the $\Psi_{ij}$ term, which equals $- n_{ij}(1 - b_i/d)(1 - p_{ij}) - Q_{ij}$, captures the influence of both local and global term statistics through a product of three factors: the number of occurrences of the $i$'th term in the $j$'th document, the proportion of documents that do not contain the $i$'th term, and the proportion of non-occurrences of the $i$'th term in the $j$'th document. This product, together with $Q_{ij}$ corrects for the discrepancy between $-\log{H_{ij}}$ and $\varhyphen{TF-IDF}(i,j)$.

\section{Numerical examples} \label{sec:numerical-examples}
This section demonstrates the validity of Theorem~\ref{theorem:hgt-tficf} and its two corollaries by way of some targeted numerical examples.

We evaluate the negative logarithm of the $p$-value from Fisher's exact test, $\varhyphen{TF-ICF}(i,j)$ $+ \Phi_{ij}$ of Eq.~(\ref{eq:hgt-tficf-approx}), $\varhyphen{TF-IDF}(i,j) + \Psi_{ij}$ of Eq.~(\ref{eq:hgt-tfidf-approx}), and $\varhyphen{TF-IDF}(i,j)$ across different parameter settings, considering both relatively small and large numbers of terms in the collection. We selected parameter values based on the conditions of Theorem~\ref{theorem:hgt-tficf}, Corollary~\ref{cor:hgt-tficf-1}, and Corollary~\ref{cor:hgt-tficf-2}.

\begin{table*}[!th]
\centering
\addtolength{\tabcolsep}{-0.05em}
\resizebox{\textwidth}{!}{
\begin{tabular}{lrrrrrr}
\multicolumn{1}{l}{} & \multicolumn{6}{c}{Small $n$ case} \\
\cmidrule(lr){2-7}
\multicolumn{1}{l}{} & \multicolumn{2}{c}{Thrm.~\ref{theorem:hgt-tficf} verification} & \multicolumn{2}{c}{Cor.~\ref{cor:hgt-tficf-1} verification}  & \multicolumn{2}{c}{Cor.~\ref{cor:hgt-tficf-2} verification} \\
& $n=1,000$ & $n_j=100$ & $n=1,000$ & $n_j=25$ & $n=1,000$ & $n_j=20$ \\
& $n_i=150$ & $n_{ij}=25$ & $n_i=100$ & $n_{ij}=10$& $n_i=160$ & $n_{ij}=20$\\
& $b_i=4$ & $d=20$ & $b_i=10$ & $d=40$ & $b_i=8$ & $d=50$ \\
\cmidrule(lr){2-3} \cmidrule(lr){4-5} \cmidrule(lr){6-7}
Formula & Result & $\lvert\Delta(\%)\rvert$ & Result & $\lvert\Delta(\%)\rvert$ & Result & $\lvert\Delta(\%)\rvert$ \\
\midrule
$-\log H_{ij}$ & 5.5429& 0.0000 & 9.7407 & 0.0000 & 37.6993 & 0.0000 \\
$\varhyphen{TF-ICF}(i,j)+\Phi_{ij}$ & 4.7111 & 17.6554 & 9.2446 & 5.3662 & 36.6516 & 2.8584 \\
$\varhyphen{TF-IDF}(i,j)+\Psi_{ij}$ & 24.6764 & 77.5378 & 9.2446 & 5.3662 & 36.6516 & 2.8584 \\
$\varhyphen{TF-IDF}(i,j)$ & 40.2359 & 86.2241 & 13.8629 & 29.7354 & 36.6516 & 2.8584 \\
\midrule
\addlinespace[5mm]
\multicolumn{1}{l}{} & \multicolumn{6}{c}{Large $n$ case} \\
\cmidrule(lr){2-7}
\multicolumn{1}{l}{} & \multicolumn{2}{c}{Thrm.~\ref{theorem:hgt-tficf} verification} & \multicolumn{2}{c}{Cor.~\ref{cor:hgt-tficf-1} verification}  & \multicolumn{2}{c}{Cor.~\ref{cor:hgt-tficf-2} verification} \\
 & $n=10,000$ & $n_j=75$& $n=10,000$ & $n_j=100$& $n=10,000$ & $n_j=80$ \\
 & $n_i=200$ & $n_{ij}=15$ & $n_i=200$ & $n_{ij}=25$ & $n_i=1,200$ & $n_{ij}=80$\\
 & $b_i=20$ & $d=75$ & $b_i=8$ & $d=100$ & $b_i=15$ & $d=125$\\
\cmidrule(lr){2-3} \cmidrule(lr){4-5} \cmidrule(lr){6-7}
Formula & Result & $\lvert\Delta(\%)\rvert$ & Result & $\lvert\Delta(\%)\rvert$ & Result & $\lvert\Delta(\%)\rvert$ \\
\midrule
$-\log H_{ij}$ & 24.8971 & 0.0000 & 46.7698 & 0.0000 & 171.9977 & 0.0000 \\
$\varhyphen{TF-ICF}(i,j)+\Phi_{ij}$ & 23.6898 & 5.0964 & 45.8791 & 1.9414 & 169.6211 & 1.4012 \\
$\varhyphen{TF-IDF}(i,j)+\Psi_{ij}$ & 10.9773 & 126.8048 & 45.8791 & 1.9414 & 169.6211 & 1.4012 \\
$\varhyphen{TF-IDF}(i,j)$ & 19.8263 & 25.5758 & 63.1432 & 25.9306 & 169.6211 & 1.4012 \\
\midrule
\end{tabular}
}
\caption{Numerical validation of theoretical relationships connecting Fisher's exact test with the TF-IDF and TF-ICF term-weighting functions.}
\label{tbl:numerical-results}
\end{table*}

Table~\ref{tbl:numerical-results} shows the results of the outlined experiments, rounded to four decimal places, and compares the absolute percentage difference, denoted as $\lvert\Delta(\%)\rvert$, between the function of interest and $-\log H_{ij}$ across various parameter settings. Columns two and three~(Thrm.~\ref{theorem:hgt-tficf} verification) illustrate the behavior of the functions when the conditions of Theorem~\ref{theorem:hgt-tficf}, exclusively, are met (e.g., documents are not of equal length). We can see that the absolute percentage difference between $\varhyphen{TF-ICF}(i,j)+\Phi_{ij}$ and $-\log H_{ij}$ decreases from 17.7\% to 5.1\% with increasing $n$, as seen in the large $n$ section compared to the small $n$ section. Columns four and five~(Cor.~\ref{cor:hgt-tficf-1} verification) show the results when the three conditions of Corollary~\ref{cor:hgt-tficf-1} are met. Similarly, the absolute percentage difference between $\varhyphen{TF-IDF}(i,j)+\Psi_{ij}$ and $-\log H_{ij}$ decreases from 5.4\% to 1.9\% as $n$ increases. Columns six and seven~(Cor.~\ref{cor:hgt-tficf-2} verification) show the results under the extreme conditions of Corollary~\ref{cor:hgt-tficf-2}. We observe that $\varhyphen{TF-IDF}(i,j)$ remains closely aligned with $-\log H_{ij}$, as indicated by a small absolute percentage difference of 2.9\% in the small $n$ section and 1.4\% in the large $n$ section. In summary, we are happy to report that the outcomes of our numerical examples validate our theoretical findings.

While the primary goal of this section is to validate our theoretical findings with numerical examples, we also examine the behavior of the derived formulas in settings which are more typical of real data where the assumptions underlying are derivations may be violated.

\begin{table*}[th]
\centering
\addtolength{\tabcolsep}{-0.05em}
\begin{tabular}{lrrrr}
\multicolumn{1}{l}{} & \multicolumn{2}{c}{Case I} & \multicolumn{2}{c}{Case II} \\
& $n=10,000$ & $n_j=75$ & $n=12,500$ & $n_j=80$ \\
& $n_i=125$ & $n_{ij}=7$ & $n_i=6$ & $n_{ij}=2$ \\
& $b_i=12$ & $d=175$ & $b_i=3$ & $d=200$ \\
\cmidrule(lr){2-3} \cmidrule(lr){4-5}
Formula & Result & $\lvert\Delta(\%)\rvert$ & Result & $\lvert\Delta(\%)\rvert$ \\
\midrule
$-\log H_{ij}$ & 10.1385 & 0.0000 & 7.4240 & 0.0000 \\
$\varhyphen{TF-ICF}(i,j)+\Phi_{ij}$ & 8.4774 & 19.5938 & 5.9860 & 24.0226 \\
$\varhyphen{TF-IDF}(i,j)+\Psi_{ij}$ & 12.7487 & 20.4740 & 6.4716 & 14.7178 \\
$\varhyphen{TF-IDF}(i,j)$ & 18.7592 & 45.9544 & 8.3994 & 11.6125 \\
\midrule
\end{tabular}
\caption{Numerical experiment evaluating the behavior of the derived formulas in settings more typical of real data.} \label{tbl:numerical-results-real}
\end{table*}

The Case~I columns of Table~\ref{tbl:numerical-results-real} show a regime where the number of documents is larger than before (i.e., $d=175$) and the number of occurrences of the $i$'th term in the $j$'th document is of a more realistically smaller size (i.e., $n_{ij}=7$). This setup violates the conditions of Theorem~\ref{theorem:hgt-tficf}, such as~$n_{ij}$ being sufficiently large, and consequently, the conditions of Corollaries~\ref{cor:hgt-tficf-1} and~\ref{cor:hgt-tficf-2}. As a result, the absolute percentage difference between Theorem~\ref{theorem:hgt-tficf}'s $\varhyphen{TF-ICF}(i,j)+\Phi_{ij}$ and $-\log{H_{ij}}$ is large compared to the results in Table~\ref{tbl:numerical-results} employing the same number of terms in the collection (i.e., $n=10,000$). The Case~II columns of the table show a regime where $i$'th term occurs in a small number of documents (i.e., $p_{i}$ and $b_i/d$ are small). Similarly, this setup violates some of the required conditions of Theorem~\ref{theorem:hgt-tficf}, such as~$n_{ij}$ being sufficiently large, and the conditions of Corollaries~\ref{cor:hgt-tficf-1} and~\ref{cor:hgt-tficf-2}, such as documents being of equal length. Although we chose a large value for the total number of terms in the collection (i.e., $n=12,500$), the absolute percentage difference is not close to zero for all target formulas due to the violation of their respective conditions.

\section{Conclusion} \label{sec:conclusion}
In this paper we have presented a novel Fisher’s exact test justification of the celebrated TF–IDF term-weighting scheme. In so doing we have established the first theoretical foundation for TF–IDF within a significance testing framework. The Fisher’s exact test justification of TF–IDF offers a new and intuitive way of understanding why it is so effective in practical applications.

This work was motivated by a desire to understand TF–IDF from the perspective of statistical inference. In demonstrating how TF–IDF is connected with Fisher’s exact test in Corollary~\ref{cor:hgt-tficf-1}, we adopt the simplifying assumptions that~(1) all documents are of equal length;~(2) the term of interest appears in the document under consideration; and~(3) in all other documents, the term appears either not at all or with the same fixed positive value as in the document of interest. While these constraints are mathematically idealized, they may serve in practice as reasonable approximations for diagnostic terms, such as programming languages like ``Python'' or ``Java'', which tend to occur with relatively consistent prevalence in relevant documents and are otherwise absent. Importantly, the direct asymptotic equivalence between TF–IDF and the Fisher's exact test negative log-transformed $p$-value of Corollary~\ref{cor:hgt-tficf-2} relies on the highly idealized assumption that the term of interest appears either exclusively or not at all in each document. As such, the clean relations presented in Corollaries~\ref{cor:hgt-tficf-1} and~\ref{cor:hgt-tficf-2} should be viewed as tenuous bridges the worlds of between term-weighting functions and significance testing, rather than general claims of TF–IDF being equivalent to Fisher's exact test.

The Fisher's exact test rationale for TF–IDF that we advance is likely rooted in the fundamental linguistic phenomenon known as word burstiness. In a document collection, word burstiness refers to a tendency for the occurrences of certain terms to concentrate in only a few documents, rather than being spread evenly across the entire collection. Word burstiness also reflects the notion that each repeated occurrence of the same word in a document becomes less surprising, and consequently, progressively less informative about the document's meaning. TF–IDF leverages word burstiness indirectly by assigning higher weight to terms that occur frequently in a document of interest (via the TF factor) but are uncommon elsewhere (via the IDF factor). Fisher's exact test leverages word burstiness in its own way by amplifying terms that constitute a high proportion of term occurrences in a document of interest (via $p_{ij}=n_{ij}/n_j$) relative to the remaining documents in the collection (via $\tilde{p}_{ij}=(n_i-n_{ij})/(n-n_j)$). Word burstiness thus provides an intuitive and linguistically grounded basis for the mathematical relationships between TF–IDF and Fisher's exact test that we have uncovered.

Future research should aim to relax the somewhat restrictive assumptions we employed. This could mean using established statistical language models to represent the underlying document collection. For example, one could explore modeling document generation according to the Dirichlet-multinomial language model, which explicitly models the phenomenon of word burstiness. By imposing certain constraints on the model parameters, it may be possible to derive TF-IDF-like term weights. A complementary approach could involve exploring generalizations of Fisher's exact test that accommodate the nonuniform sampling of terms. One option is a significance test built on the multivariate Wallenius noncentral hypergeometric distribution~\citep[see][]{Garcia2022}. For this test, the terms $t_1,\ldots,t_m$ may be associated with Wallenius sampling weights $\omega_1 = n_1^\alpha,\ldots,\omega_m = n_m^\alpha$ where $0<\alpha\leq 1$ is a smoothing parameter controlling for the influence of rare terms. Since IDF (via $d/b_i$) places a higher weight on rare terms than ICF (via $n/n_i$), exploring weighted sampling models could offer additional insights into a connection between TF–ICF to TF–IDF.  Bayesian extensions of Fisher's exact test, as discussed in~\cite{Cao2017}, offer an alternative path forward.

Another interesting direction concerns rooting other prominent TF–IDF family members in the theory of significance testing. This would aid in consolidating the dizzying array of term-weighting heuristics into a unified statistical framework. Besides being an interesting exercise in itself, a map of the term-weighting terrain could serve as a valuable guide to practitioners in their search for yet more effective schemes.

We end on some practical implications of our work. As we have seen, the development of significance tests for functional enrichment is a flourishing area of research in functional genomics. In text analysis, by contrast, application of the significance testing paradigm to term-weighting scheme design remains underexplored. This is intriguing because the paradigm is ideally suited for quantifying term weight as readily interpretable $p$-values, unlike with traditional term-weightings where scores must be interpreted on a scheme-by-scheme basis. Moreover, $p$-values, being tail probabilities, are potentially more informative than comparable term-weightings, which are typically point values rather than sums. It is our hope that the demonstration of a connection between TF–IDF and Fisher's exact test will spur further activity in the application of significance tests in text analysis.

\if1\blind
{
\section*{Acknowledgment}
The authors thank Rong Hu for conducting a preliminary mathematical investigation. The authors also thank the two anonymous reviewers and the editor for their thoughtful and constructive comments, which have very much improved the quality of the manuscript. The work was supported by the Natural Science and Engineering Research Council of Canada [Discovery and Alliance Grants to A.A.F.].

\section*{Declaration}
The authors report there are no competing interests to declare.

\section*{Code availability}
Code used to produce the numerical examples in this work is available at the GitHub repository \url{https://github.com/sheridan-stable/from-fisher-to-tfidf}, Release v2024.12.

\section*{Author contribution statement}
\textbf{Paul Sheridan}: Conceptualization, Numerical experiments, Proofs, Results interpretation, Supervision, Writing – original draft, Writing – and review \& editing. \textbf{Zeyad Ahmed}: Computer code, Numerical experiments, Proofs, Results interpretation, Writing – original draft, Writing – and review \& editing. \textbf{Aitazaz A. Farooque}: Funding acquisition, Results interpretation, Supervision, Writing – review \& editing. All authors reviewed the results and approved the final version of the manuscript. 
} \fi

\bibliographystyle{chicago}
\bibliography{bibliography}

\end{document}